\documentclass[10pt,twocolumn,letterpaper]{article}
\usepackage{algorithm,algorithmic}
\usepackage{listings}
\usepackage{cvpr}
\usepackage{times}
\usepackage{epsfig}
\usepackage{graphicx}
\usepackage{amsmath,amssymb,mathtools}
\usepackage{bm}
\usepackage{here}
\usepackage{setspace}
\usepackage{latexsym}
\usepackage{mediabb}

\def\vec#1{{\bf{#1}}}

\newtheorem{theorem}{Theorem}[section]
\newtheorem{lemma}[theorem]{Lemma}
\newtheorem{proposition}[theorem]{Proposition}
\newtheorem{corollary}[theorem]{Corollary}

\newenvironment{proof}[1][Proof]{\begin{trivlist}
\item[\hskip \labelsep {\bfseries #1}]}{\end{trivlist}}

\newcommand{\qed}{\nobreak \ifvmode \relax \else
      \ifdim\lastskip<1.5em \hskip-\lastskip
      \hskip1.5em plus0em minus0.5em \fi \nobreak
      \vrule height0.75em width0.5em depth0.25em\fi}

\usepackage[pagebackref=true,breaklinks=true,letterpaper=true,colorlinks,bookmarks=false]{hyperref}

\cvprfinalcopy 

\def\cvprPaperID{0} 

\ifcvprfinal\fi
\begin{document}

\title{Fast Supervised Discrete Hashing and its Analysis}

\author{Gou Koutaki\\
Kumamoto University\\
{\tt\small koutaki@cs.kumamoto-u.ac.jp}
\and
Keiichiro Shirai\\
Shinshu University\\
{\tt\small keiichi@shinshu-u.ac.jp}
\and
Mitsuru Ambai\\
Denso IT Laboratory\\
{\tt\small manbai@d-itlab.co.jp}
}

\maketitle

\begin{abstract}
In this paper, we propose a learning-based supervised discrete hashing method. Binary hashing is widely used for large-scale image retrieval as well as video and document searches because the compact representation of binary code is essential for data storage and reasonable for query searches using bit-operations. The recently proposed Supervised Discrete Hashing (SDH) efficiently solves mixed-integer programming problems by alternating optimization and the \textit{Discrete Cyclic Coordinate descent} (DCC) method. We show that the SDH model can be simplified without performance degradation based on some preliminary experiments; we call the approximate model for this the ``Fast SDH'' (FSDH) model. We analyze the FSDH model and provide a mathematically exact solution for it. In contrast to SDH, our model does not require an alternating optimization algorithm and does not depend on initial values. FSDH is also easier to implement than Iterative Quantization (ITQ). Experimental results involving a large-scale database showed that FSDH outperforms conventional SDH in terms of precision, recall, and computation time.
\end{abstract}

%
%
%

\section{Introduction}

Binary hashing is an important technique for computer vision, machine learning, and large-scale image/video/document retrieval \cite{Gionis:1999:LSH, Gong:2012_ITQ_TPAMI, BRE_NIPS2009, HashICML13a, Nguyen:2014:SDH, Shen_2015_ICCV, Shen_2015_CVPR}. Through binary hashing, multi-dimensional feature vectors with integers or floating-point elements are transformed into short binary codes.
This representation of binary code is an important technique since large-scale databases occupy large amounts of storage.
Furthermore, it is easy to compare a query in binary code with a binary code in a database because the Hamming distance between them can be computed efficiently by using bitwise operations that are part of the instruction set of any modern CPU \cite{Calonder:2010:BBR,Gog:2016:FCH}.

Many binary hashing methods have been proposed. \textit{Locality-sensitive hashing} (LSH) \cite{Gionis:1999:LSH} is one of most popular methods. In LSH, binary codes are generated by using a random projection matrix and thresholding using the sign of the projected data. \textit{Iterative quantization} (ITQ) \cite{Gong:2012_ITQ_TPAMI} is another state-of-the-art binary hashing method. 
In ITQ, a projection matrix of the hash function is optimized by iterating projection and thresholding procedures according to the given training samples.

Binary hashing can be roughly classified into two types: unsupervised hashing \cite{BRE_NIPS2009, HGraph:ICML11, HGraph:NIPS2014, Shen_2015_ICCV, HEO:SH_TPAMI2015, Weiss:SpectralH_NIPS2008} and supervised hashing.
Supervised hashing uses learning label information if it exists. 
In general, supervised hashing yields better performance than unsupervised hashing, so in this study, we target supervised hashing.
In addition, some unsupervised methods such as LSH and ITQ can be converted into supervised methods by imposing label information on feature vectors. For example, canonical correlation analysis (CCA) \cite{CCA} can transform feature vectors to maximize inter-class variation and minimize intra-class variation according to label information. Hereafter, we call these processes CCA-LSH and CCA-ITQ, respectively.

Not imposing label information on feature vectors, such as in CCA, but imposing it directly on hash functions has been proposed. \textit{Kernel-based supervised hashing} (KSH) \cite{CVPR12:KSH} uses spectral relaxation to optimize the cost function through a sign function. Feature vectors are transformed by kernels during preprocessing. KSH has also been improved to 
\textit{kernel-based supervised discrete hashing} (KSDH) \cite{ECCV15:KSDH}. It relaxes the discrete hashing problem through linear relaxation. 
\textit{Supervised Discriminative Hashing} \cite{Nguyen:2014:SDH} decomposes training samples into inter and intra samples. \textit{Column sampling-based discrete supervised hashing} (COSDISH) \cite{COSDISH:AAAI16} uses column sampling based on semantic similarity, and decomposes the problem into a sub-problem to simplify solution.

The optimization of binary codes leads to a mixed-integer programming problem involving integer and non-integer variables, which is an NP-hard problem in general \cite{Shen_2015_CVPR}. Therefore, many methods discard the discrete constraints, or transform the problem into a relaxed problem, i.e., a linear programming problem \cite{Schrijver:1986}. This relaxation significantly simplifies the problem, but is known to affect classification performance \cite{Shen_2015_CVPR}.

Recent research has introduced a type of \textit{supervised discrete hashing} (SDH) \cite{Shen_2015_CVPR,Wang_2016_CVPR} that directly learns binary codes without relaxation. SDH is a state-of-the-art method because of its ease of implementation, reasonable computation time for learning, and better performance over other state-of-the-art supervised hashing methods. To solve discrete problems, SDH uses a \textit{discrete cyclic coordinate descent} (DCC) method, which is an approximate solver of 0-1 quadratic integer programming problems. 

%
%

\subsection{Contributions and advantages}

In this study, we first analyze the SDH model and point out that it can be simplified without performance degradation based on some preliminary experiments. 
We call the approximate model the \textit{fast SDH} (FSDH) model. We analyze the FSDH model and provide a mathematically exact solution to it. The model simplification is validated through experiments involving several large-scale datasets.

The advantages of the proposed method are as follows:
\begin{itemize}
\setlength{\topsep}{0pt}
\setlength{\parskip}{0.5\baselineskip}
\setlength{\itemsep}{0pt}
\item Unlike SDH, it does not require alternating optimization or hyper-parameters, and is not initial value-dependent.
\item It is easier to implement than ITQ and is efficient in terms of computation time. 
FSDH can be implemented in three lines on MATLAB.
\item High bit scalability: its learning time and performance do not depend on the code length.
\item It has better precision and recall than other state-of-the-art supervised hashing methods.
\end{itemize}
%

%
%

\subsection{Related work}

As described subsequently, the SDH model poses a matrix factorization problem: $\vec{F}=\vec{W}^{\top}\vec{B}$. 
The popular form of this problem is singular value decomposition (SVD) \cite{Golub:1996}, 
and when $\vec{W}$ and $\vec{B}$ are unconstrained, the Householder method is used for computation. 
When $\vec{W} \geq 0$, non-negative matrix factorization (NMF) is used \cite{NFM:NPIS2005}.

In the case of the SDH model, $\vec{B}$ is constrained to $\{-1,1\}$ and $\vec{W}$ 
is unconstrained. In a similar problem setting, Slawski \textit{et al.} proposed matrix 
factorization with binary components \cite{SlawskiHL13} and showed an application to DNA analysis for cancer research. 
$\vec{B}$ is constrained to $\{0,1\}$, and indicates Unmethylated/Methylated DNA sequences. 
Furthermore, a similar model has been proposed in display electronics. 
Koutaki proposed binary continuous decomposition for multi-view displays \cite{Koutaki:2016}. 
In this model, multiple images $\vec{F}$ are decomposed into binary images $\vec{B}$ 
and a weight matrix $\vec{W}$. An image projector projects binary 0-1 patterns through digital mirror devices (DMDs), 
and the weight matrix corresponds to the transmittance of the LCD shutter.

%
%
%

\section{Supervised Discrete Hashing (SDH) Model}

In this section, we introduce the \textit{supervised discrete hashing} (SDH) 
model.
Let $\vec{x}_i \!\in\! \mathbb{R}^M$ be a feature vector, and introduce a set of $N\, (\ge\! M)$ training samples $\vec{X} = [\vec{x}_1,\ldots,\vec{x}_N] \in \mathbb{R}^{M \times N}$.
Then, consider binary label information $\vec{y}_i\!\in\!\{0,1\}^{C}$ corresponding to $\mathbf{x}_i$, where $C$ is the number of categories to 
classify. Setting the $k$-th element to 1, $[\vec{y}_i]_k = 1$, and the other elements to 0 indicates that the $i$-th vector belongs to class $k$. 
By concatenating $N$ samples of $\vec{y}_i$ horizontally, 
a label matrix  $\vec{Y}=\left[\vec{y}_1,\ldots,\vec{y}_N\right] \in \{0,1\}^{C \times N}$ is constructed.

%
%

\subsection{Binary code assignment to each sample}

For each sample $\vec{x}_i$, an $L$-bit binary code $\vec{b}_i \!\in\! \{-1,1\}^L$ is assigned.
By concatenating $N$ samples of $\vec{b}_i$ horizontally, a binary matrix 
$\vec{B}=[\vec{b}_1,\ldots,\vec{b}_N] \in \{-1,1\}^{L \times N}$
is constructed.
The binary code $\vec{b}_i$ is computed as
\begin{equation}
\vec{b}_i = \text{sgn}\left(\vec{P}^\top \vec{x}_i \right),
\label{eq:Px}
\end{equation}
where $\vec{P} \!\in\! \mathbb{R}^{M \times L}$
(therefore $\vec{P}^\top \!\in\! \mathbb{R}^{L \times M}$) is a linear transformation matrix and $\text{sgn}(\cdot)$ is the sign function. The major aim of SDH is to determine the matrix $\vec{P}$ from training samples $\vec{X}$. In practice, feature vectors $\{\vec{x}_i\}$ are transformed by preprocessing. Therefore, we denote the original feature vectors $\vec{x}^{ori}_i$ and the transformed feature vectors $\vec{x}_i$.

%
%

\subsection{Preprocessing: Kernel transformation}

The original feature vectors of training samples 
$\vec{x}^{ori}_i\ (i=1,\ldots,N)$ 
are converted into the feature vectors 
$\vec{x}_i \in \mathbb{R}^{M}$ using the following 
kernel transformation $\Phi$:
\begin{equation}
\scalebox{0.9}{$
\begin{aligned}
& \vec{x}_i = \Phi(\vec{x}^{ori}_i)\\
&= \left[
\exp\! \left(-\frac{\|\vec{x}^{ori}_i \! - \vec{a}_1 \|^2}{\sigma}\right),
\ldots,
\exp\! \left(-\frac{\|\vec{x}^{ori}_i \! - \vec{a}_m \|^2}{\sigma}\right)
\right]^\top,
\end{aligned}
$}
\end{equation}
where $\vec{a}_m$ is an anchor vector obtained by randomly 
sampling the original feature vectors, 
$\vec{a}_m = \vec{x}^{ori}_{rand}$. 
Then, the transformed feature vectors are bundled into 
the matrix form  $\vec{X}=[\vec{x}_1,\ldots,\vec{x}_N]$.

%
%

\subsection{Classification model} \label{subsec:class}

Following binary coding by \eqref{eq:Px}, we suppose that a good binary code classifies the class, and formulate the following simple linear classification model:
\begin{equation}
\widehat{\vec{y}}_i = \vec{W}^\top \vec{b}_i,
\end{equation}
where $\vec{W}\!\in\!\mathbb{R}^{L \times C}$ is a weight matrix and $\widehat{\vec{y}}_i$ is an estimated label vector. As mentioned above, its maximum index, $\arg\min_k [\widehat{\vec{y}}_i]_k$, indicates the assigned class of $\vec{x}_i$.

%
%

\subsection{Optimization of SDH}

The SDH problem is defined as the following minimization problem:
\begin{equation}
\scalebox{0.95}{$
\begin{aligned}
\min_{\vec{B},\vec{W},\vec{P}} \|\vec{Y}-\vec{W}^\top \vec{B}\|^2 + \lambda \|\vec{W}\|^2 + \nu \|\vec{B}-\vec{P}^\top \vec{X}\|^2,
\label{eq:SDH}
\end{aligned}
$}
\end{equation}
where $\|\cdot\|$ is the Frobenius norm, and $\lambda \geq 0$ and $\nu \geq 0$ are balance parameters.
The first term includes the classification model explained in Sec.~\ref{subsec:class}.
The second term is a regularizer for $\vec{W}$ to avoid overfitting.
The third term indicates the fitting errors due to binary coding.

In this optimization, it is sufficient to compute $\vec{P}$, \textit{i.e.}, if $\vec{P}$ is obtained, 
$\vec{B}$ can be obtained by \eqref{eq:Px}, and
$\vec{W}$ can be obtained from the following simple least squares equation:
\begin{equation}
\vec{W}=\left(\vec{B}\vec{B}^\top + \lambda \vec{I}\right)^{-1}\vec{B}\vec{Y}^\top.
\label{eq:W}
\end{equation}
However, due to the difficulty of optimization, the optimization problem of \eqref{eq:SDH} is usually divided into three sub-problems of the optimization of $\vec{B}, \vec{W}$, and $\vec{P}$. Thus, the following alternating optimization is performed:

\vspace{0.5\baselineskip}
\noindent(i) \textbf{Initialization}:
$\vec{B}$ is initialized, usually randomly.

\vspace{0.5\baselineskip}
\noindent(ii) \textbf{F-Step}:
$\vec{P}$ is computed by the following simple least squares method:
\begin{equation}
\vec{P}=\left(\vec{X}\vec{X}^\top\right)^{-1}\vec{X}\vec{B}^\top.
\label{eq:P}
\end{equation}
%

\vspace{0.5\baselineskip}
\noindent(iii) \textbf{W-Step}:
$\vec{W}$ is computed by \eqref{eq:W}.

\vspace{0.5\baselineskip}
\noindent(iv) \textbf{B-Step}:
After fixing $\vec{P}$ and $\vec{W}$, equation~\eqref{eq:SDH} becomes:
\begin{equation}
\begin{aligned}
\min_{\vec{B}} \ & \|\vec{Y}\|^2 - 2\mathrm{Tr}\left(\vec{Y}\vec{W}^\top\vec{B}\right)
+\mathrm{Tr}\left(\vec{B}^\top\vec{W}\vec{W}^\top\vec{B}\right) \\
&+ \nu \left(\|\vec{B}\|^2 - 2\mathrm{Tr}\left(\vec{P}^\top\vec{X}\vec{B}\right) + \|\vec{P}^\top\vec{X}\|^2\right)\\
\Rightarrow \ &\min_{\vec{B}} \mathrm{Tr}\left(\vec{B}^\top\vec{Q}\vec{B} + \vec{F}^\top\vec{B}\right),
\end{aligned}
\end{equation}
where
\begin{equation}
\vec{Q}\!=\!\vec{W}\vec{W}^\top \!\in\! \mathbb{R}^{L\times L}, \
\vec{F}\!=\!-2\left(\vec{W}\vec{Y}+\nu\vec{P}^\top\vec{X}\right) \!\in\! \mathbb{R}^{L \times N}.
\label{eq:F}
\end{equation}
Note that $\mathrm{Tr}\left(\vec{B}^\top\vec{B}\right) = LN$. The trace can be rewritten as
\begin{equation}
\min_{\{\vec{b}_i\}} \ \sum_{i=1}^N \vec{b}_i^\top\vec{Q}\vec{b}_i + \vec{f}_i^\top\vec{b}_i,
\end{equation}
where $\vec{f}_i \in \mathbb{R}^L$ is the $i$-th column vector of $\vec{F}$. $\{\vec{b}_i\}$ are actually independent of one another. Therefore, it reduces to the following 0-1 integer quadratic programming problem for each $i$-th sample:
\begin{equation}
\forall_i \ \min_{\vec{b}_i \in \{-1,1\}^L} \vec{b}_i^\top\vec{Q}\vec{b}_i + \vec{f}_i^\top\vec{b}_i.
\label{eq:qeq}
\end{equation}

\paragraph{(v)} Iterate steps (ii)$\sim$(iv) until convergence.

%
%
%

\section{Discussion of the SDH Model}

%
%

\subsection{0-1 integer quadratic programming problem}

\paragraph{DCC method} To solve \eqref{eq:qeq}, 
SDH uses a \textit{discrete cyclic coordinate descent} (DCC) method. In this method, a one-bit element of $\vec{b}_i$ is optimized while fixing the other $L-1$ bits; the $l$-th bit $b_l$ is optimized as
\begin{equation}
b_l = -\text{sgn}\biggl( 2 \sum_{i \neq l} Q_{i,l}b_i + f_l \biggr).
\end{equation}
Then, all bits $l=1,\ldots,L$ are optimized, and this procedure is repeated several times.
In addition, the DCC method is prone to result in a local minimum because of its greediness.
To improve it, Shen \textit{et al.} proposed using a proximal operation of convex optimization \cite{TIP2016binary}.

\paragraph{Branch-and-bound method}
In the case of a large number of bits $L \geq 32$, solving \eqref{eq:qeq} exactly is difficult because this problem is NP-hard. 
However, there exist a few efficient methods to solve the 0-1 
integer quadratic programming problem. 
In \cite{Koutaki:2016}, Koutaki used 
a \textit{branch-and-bound} method to solve the problem. 
$\vec{b}$ is expanded into a binary tree of depth $L$, and 
the problem of \eqref{eq:qeq} is divided into 
a sub-problem by splitting $\vec{b}=[\vec{b}_1^\top,\vec{b}_2^\top]^\top$. At each node, the lower bound is computed and compared with the given best solution; child nodes can be excluded from the search.

The computation of the lower bound depends on the structure of $\vec{Q}, \vec{q}$, and $\vec{b}$. To compute the lower bound in general, 
the linear relaxation method is a standard method, 
$\vec{b}\in \{-1,1\}^L \Rightarrow \vec{b}\in [-1,1]^L$.  In this case, the rough lower bound of the quadratic term in \eqref{eq:qeq} can be provided by the minimum eigenvalues of $\vec{Q}$. However, linear relaxation is useless in the SDH model because $L > C$ in general, so the matrix $\vec{Q}=\vec{W}\vec{W}^\top$ is rank deficient and, as a result, the minimum eigenvalue of $\vec{Q}$ becomes zero.

Even if we can obtain an efficient algorithm, such as branch-and-bound and good lower bound, in the application of binary hashing, we still suffer from computational difficulties because code lengths $L=64, 128$, or $256$ bits are still too long to optimize, and they are used frequently.
\begin{figure}[t]
\centering
\includegraphics[width=7cm]{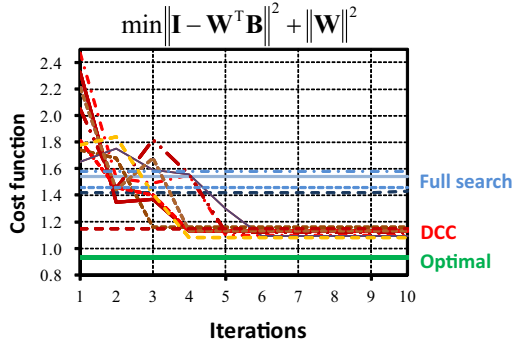}
\caption{Convergence of optimizations with several initial conditions. Even though the problem is simple, conventional solvers with alternating optimization (DCC and full search) cannot reach the optimal solution (green line) and fall into local minima.}
\label{fig:Optimal}
\end{figure}

%
%

\subsection{Alternating optimization and initial value dependence}

Even if we optimize the binary optimization in \eqref{eq:qeq}, the resulting binary codes $\vec{B}$ are not always optimal ones because they depend on the other fixed variables 
$\vec{W}$ and $\vec{P}$.
In addition, alternating optimization is prone to cause a serious problem: a solution depends on the initial values, and may fall in a local minimum during the iterations, even if each step of \textbf{F-Step}, \textbf{W-Step} and \textbf{B-Step} provides the optimal solution.

Figure~\ref{fig:Optimal} shows an example of the optimization result for a simple version of the SDH model in \eqref{eq:SDH} 
with a small number of bits $(L,C,N) \!=\! (16,10,10)$. 
In this case, an exact solution is known and its minimum value is $0.94$ (green line in Fig.~\ref{fig:Optimal}). DCC (red lines) provides results for 10 randomized initial conditions. The full search (blue lines) provides the results of an exact full search in \textbf{B-Step}, 
where $2^{16} = 65,536$ nodes are searched.

In spite of the small size of the problem, the cost function of conventional alternating solvers (DCC and full search) cannot find the exact value, and depends on initial values.
Interestingly, the results of full search immediately fall into a local minimum, and are worse than those of DCC.

%
%
%

\section{Proposed Fast SDH Model}
We introduce a new hashing model by approximating the SDH model, which utilizes the following assumptions:
\begin{description}
\setlength{\topsep}{0pt}
\setlength{\parskip}{0.7\baselineskip}
\setlength{\itemsep}{0pt}
\item[A1:] The number of bits $L$ of the binary code is a power of 2: $L = 2^l$.
\item[A2:] The number of bits is greater than the number of classes: $L \geq C$.
\item[A3:] Single-labeling problem.
\item[A4:] $\|\vec{W}^\top\vec{Y}\|^2 \gg \nu \|\vec{P}^\top\vec{X}\|^2$ in \eqref{eq:F}.
\end{description}
Note that assumptions \textbf{A1}$\sim$\textbf{A3} also become the limitations of the proposed model. In \textbf{A4}, 
SDH recommends that the parameter $\nu$ be set to a very small value, such as $\nu = 10^{-5}$ \cite{Shen_2015_CVPR}. 
In practice,  $\|\vec{W}^\top\vec{Y}\|^2 \fallingdotseq 31.53$ and 
$\nu \|\vec{P}^\top\vec{X}\|^2 \fallingdotseq 0.013$ in the CIFAR-10 dataset. Furthermore, when $\nu=0$, almost the same results can be obtained in all datasets as shown in the experimental results in Sec.~\ref{sec:results}.
We call this approximation using $\nu=0$ the ``\textit{fast SDH} (FSDH) approximation.''

Using the FSDH approximation, we solve the following problem 
for each $N$-sample $\vec{b}_i$ in \textbf{B-Step}:
\begin{equation}
\begin{aligned}
\forall_i \ & \min_{\vec{b}_i \in \{-1,1\}^L} \ \vec{b}_i^\top\vec{Q}\vec{b}_i + \vec{f}_i^\top\vec{b}_i,\\
& \vec{Q}=\vec{W}\vec{W}^\top, \quad \vec{F}=-2\vec{W}^\top\vec{Y},
\end{aligned}
\label{eq:fqeq}
\end{equation}
where $\vec{Q}$ is a constant matrix and $\vec{f}_i$ depends on label $\vec{y}_i$. By using the single-label assumption  in \textbf{A3}, 
the number of kinds of $\vec{y}_i$ is limited to $C$: 
\begin{equation}
\vec{y}_1 = [1,0,\ldots,0]^\top, \ldots, \vec{y}_C = [0,0,\ldots,1]^\top.
\label{eq:labelapprox}
\end{equation}
Thus, it is sufficient to solve only $C$ integer quadratic programming problems of \eqref{eq:fqeq} from $N$. 
In general, the number of samples $N$ is larger than that of classes: $N \gg C$, \textit{e.g.}, $N=59,000$ and $C=10$. Thereby, the computational cost of \textbf{B-Step} becomes $5,900$ times lower.
In other words, the FSDH approximation proposes the following:

\begin{proposition}
The FSDH approximation defines the SDH model to assign a binary code to each class.
\label{prop:FSDH}
\end{proposition}

After obtaining the binary codes of each class 
$\vec{B}^\prime = \left[\vec{b}^\prime_1,\ldots,\vec{b}^\prime_C\right] \in \{-1,1\}^{L \times C}$, 
the binary codes of all samples $\vec{B}$ can be 
constructed by lining up $\vec{b}_i^\prime$ as
\begin{equation}
\vec{B} = \left[\vec{b}^\prime_{y_1},\ldots,\vec{b}^\prime_{y_N}\right].
\end{equation}
After constructing $\vec{B}$, the projection 
matrix $\vec{P}$ can be obtained by \eqref{eq:P}.

%
%

\subsection{Analytical solutions of FSDH model}

From Proposition~\ref{prop:FSDH}, we found that it is sufficient to determine the binary code for each class. 
Furthermore, we can choose the optimal binary codes under the FSDH approximation as follows:
\begin{lemma}
If $f(x_i)$ is convex, the solution of
\begin{equation}
\min_{\{x_i\}} \sum_i^N f(x_i) \quad s.t. \ \sum_i^N x_i=L
\end{equation}
is given by the mean value $x_i = L / N\ (i=1,\ldots,N)$.
\label{lemma:conven}
\end{lemma}
\begin{proof}
See Appendix~\ref{sec:appendix_Lemma}.
\end{proof}

\begin{theorem}
An analytical solution of FSDH $\vec{B}^\prime$ is obtained as a Hadamard matrix.
\end{theorem}
\begin{proof}
Using the FSDH approximation and label representations in \eqref{eq:labelapprox}, the SDH model in \eqref{eq:SDH} becomes
\begin{equation}
\min_{\vec{B}^\prime,\vec{W}} \|\vec{I} - \vec{W}^\top\vec{B}^\prime\|^2 + \lambda \|\vec{W}\|^2,
\label{eq:FSDH}
\end{equation}
where $\mathbf{I} \!\in\! \mathbb{R}^{C\times C}$ is an identity matrix.
Using the solution of \eqref{eq:FSDH}, \textit{i.e.}, $\vec{W}=\left(\vec{B}^\prime\vec{B}^{\prime\top}+\lambda \vec{I}\right)^{-1}\vec{B}^\prime$, and the eigen-decomposition of 
$\vec{B}^{\prime\top}\vec{B}^\prime = \vec{P}^\top \vec{D} \vec{P}$,
we denote the eigenvalues as 
$\text{diag}(\vec{D}) \!=\! \{\sigma_i\}_{i = 1}^C$ and then get 
$\sum_{i=1}^C \sigma_i \!=\! \mathrm{Tr}(\vec{D}) \!=\! \mathrm{Tr}(\vec{B}^{\prime\top}\vec{B}^\prime) \!=\! LC$ 
as the trace of diagonal values. 
Then, equation \eqref{eq:FSDH} can be represented simply as
\begin{equation}
\min_{\vec{B}^\prime} \sum_{i=1}^C \frac{\lambda}{\sigma_i + \lambda} \quad s.t. \ \sum_{i=1}^C \sigma_i = LC. 
\end{equation} 

By lemma~\ref{lemma:conven}, $\sigma_i = L \ (i=1,\ldots,C)$. 
This implies that $\vec{B}^\prime$ is an orthogonal matrix with binary elements $\{-1,1\}$; 
in other words, $\vec{B}^\prime \in \{-1,1\}^{L \times C}$ can be given by a submatrix of the Hadamard matrix $\vec{H} \in \{-1,1\}^{L \times L}$.
$\square$
\end{proof}

\begin{corollary}
The following characteristics can be obtained easily:
\begin{itemize}
\setlength{\topsep}{0pt}
\setlength{\parskip}{0.5\baselineskip}
\setlength{\itemsep}{0pt}
\item $\vec{B}^\prime$ is independent of regularization parameter $\lambda$ ($\lambda$-invariant).
\item The optimal weight matrix $\vec{W}$ of FSDH is given by the version of the scaled binary matrix 
$\vec{B}^\prime {\rm :}\ \vec{W}=\frac{1}{L+\lambda}\vec{B}^\prime$.
\item The minimum value of \eqref{eq:FSDH} is given by $\frac{L}{L+\lambda}$.
\end{itemize}
\end{corollary}

In short, we can eliminate the \textbf{W-Step}, the alternating procedure, and the initial value dependence.  An exact solution of the FSDH model can be obtained independent of the hyper-parameters $\lambda$ and $\nu$.

%
%

\subsection{Implementation of FSDH}

Algorithm~\ref{alg1} and Figure~\ref{fig:code}, respectively, show the algorithm of FSDH and sample MATLAB code, which is simple and easy to implement. 
Figure~\ref{fig:B} shows an example of $\vec{B}^\prime$ and $\vec{B}$. A Hadamard matrix of size $2^k \!\times\! 2^k$ can be constructed recursively by Sylvester's method \cite{Sylvester:1867} as
\begin{equation}
\begin{aligned}
\vec{H}_2&=\left[
\begin{matrix*}[r]
1 & 1\\
1 & -1\\
\end{matrix*}
\right],\\
\vec{H}_{2^k}&=\left[
\begin{matrix*}[r]
\vec{H}_{2^{k-1}} & \vec{H}_{2^{k-1}}\\
\vec{H}_{2^{k-1}} & -\vec{H}_{2^{k-1}}\\
\end{matrix*}
\right] \quad (k \geq 2).
\end{aligned}
\end{equation}
Furthermore, Hadamard matrices of orders 12 and 20 
were constructed by Hadamard transformation \cite{Hadamard:1893}. Fortunately, in applications of binary hashing, 
since $L=16, 32, 64, 128, 256$, and $512$ bits are used frequently, Sylvester's method suffices in most cases. 
\begin{algorithm}[t]
\renewcommand{\algorithmicrequire}{\textbf{Input:}}
\renewcommand{\algorithmicensure}{\textbf{Output:}}
\renewcommand{\algorithmicprint}{\textbf{break}}
\caption{Fast Supervised Discrete Hashing (FSDH)}
\label{alg1}
\begin{algorithmic}[1]
\REQUIRE Pre-processed training data $\vec{X}$ and labels $\{y_i\}_{i=1}^N$:
code length $L$, number of samples $N$, number of classes $C$.
\ENSURE Projection matrix $\vec{P}$.
\STATE Compute Hadamard matrix $\vec{H} \in \{-1,1\}^{L \times L}$
\STATE Let $\left[\vec{b}_1^\prime,\ldots, \vec{b}_C^\prime \right]$ be $C$ columns of $\vec{H}$.
\STATE Construct $\vec{B}$ by $\vec{b}_i = \vec{b}^\prime_{y_i}$.
\STATE Compute $\vec{P}$ from $\vec{B}$ and $\vec{X}$ by \eqref{eq:P}.
\end{algorithmic}
\end{algorithm}
%

%

%

%
%

\subsection{Analysis of bias term of FSDH}

We have already shown that $\vec{B}$ obtained from the Hadamard matrix minimizes two terms: $\|\vec{Y}-\vec{W}^\top\vec{B}\|^2+\lambda\|\vec{W}\|^2$. Furthermore, we pay attention to how $\vec{B}$ affects the bias term $\|\vec{B}-\vec{P}^\top\vec{X}\|^2$. In this subsection, we continue to analyze its behavior. We suppose that samples are sorted by label $y_i$. Let $\vec{P}^\top=\vec{B}\vec{X}^\top\left(\vec{X}\vec{X}^\top\right)^{-1}$ be the bias term:
\begin{equation}
\begin{aligned}
\|\vec{B}-\vec{P}^\top\vec{X}\|^2
&=\|\vec{B}\left(\vec{I}-\vec{K}\right)\|^2\\
&=\mathrm{Tr}\left(\vec{B}^\top\vec{B}\right)-\mathrm{Tr}\left(\vec{B}\vec{K}\vec{B}^\top\right),
\end{aligned}
\end{equation}
\begin{figure}
\begin{lstlisting}[basicstyle=\ttfamily\footnotesize, frame=single]
HA = hadamard(L);  %L-bit Hadamard matrix
B  = HA(y,:);      %y:label array
P  = (X*X')\(X*B');%X:feature vectors
\end{lstlisting}
\caption{Sample MATLAB code for the main part of FDSH, which is implemented in only three lines. It is easier to implement than the ITQ algorithm.}
\label{fig:code}
\end{figure}
\begin{figure}[t]
\centering
\includegraphics[width=8cm]{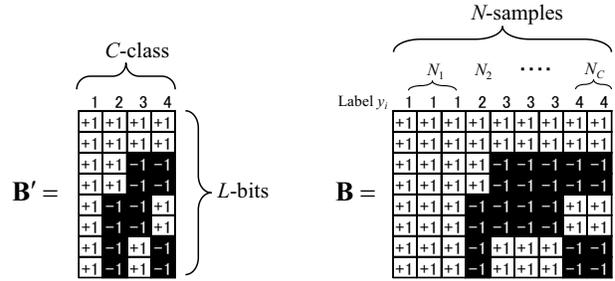}
\caption{An example of the construction of $\vec{B}^\prime$  and $\vec{B}$ with $L=8$ bits and $C=4$ classes.  After computing $\vec{B}^\prime$ from the Hadamard matrix, binary code $\vec{B}$ with $N$ columns is constructed according to label $y_i$.}
\label{fig:B}
\end{figure}
where $\vec{K} \!=\! \vec{X}^\top\left(\vec{X}\vec{X}^\top\right)^{-1}\vec{X} \!\in\! \mathbb{R}^{N \times N}$ is a projection matrix.
Therefore, to reduce the bias term, it is better that $\mathrm{Tr}\left(\vec{B}\vec{K}\vec{B}^\top\right)$ has a large value. Then, using $\vec{K}=\vec{K}\vec{K}$, we can rewrite it as
\begin{equation}
\begin{aligned}
\mathrm{Tr}\left(\vec{B}\vec{K}\vec{B}^\top\right) =
\mathrm{Tr}\left(\vec{K}\vec{B}^\top\vec{B}\vec{K}\right),
\end{aligned}
\label{eq:trace_eq}
\end{equation}
where $\vec{B}^\top\vec{B}$ is a block-diagonal matrix
\begin{equation}
\begin{aligned}
\vec{B}^\top\vec{B}&=
L\left[
\begin{matrix*}[l]
\vec{J}_{N_1} &                & \mathbf{O}  \\
              &        \ddots  &   \\
\mathbf{O}    &                & \vec{J}_{N_C}\\
\end{matrix*}
\right],
\end{aligned}
\end{equation}
%
$\vec{J}_{N_k} \in 1^{N_k\times N_k}$ are matricies with all elements equal to 1, and $N_k$ is the number of samples with label $y_i = k$. Using these values, $\mathrm{Tr}\left(\vec{K}\vec{B}^\top\vec{B}\vec{K}\right) $ in \eqref{eq:trace_eq} can be expressed as
\begin{equation}
\begin{aligned}
L \sum_{i=1}^N & \left[\left(K_{i,1}+\ldots+K_{i,N_1}\right)^2 + \left(K_{i,N_1+1}+\ldots+K_{i,N_2}\right)^2 \right.\\
&\left. +\ldots+\left(K_{i,N_{C-1}+1}+\ldots+K_{i,N_C}\right)^2\right],
\end{aligned}
\end{equation}
where $\{K_{i,j}\}$ with the same label $y_i=y_j$ are summed up.
Since the definition of $\vec{K}$ is $\vec{K}=\vec{X}^\top\left(\vec{X}\vec{X}^\top\right)^{-1}\vec{X}$, $K_{ij}$ can be regarded as the normalized correlation of $\vec{x}_i$ and $\vec{x}_j$. Since samples with the same label must represent a similar feature vector, $\mathrm{Tr}\left(\vec{K}\vec{B}^\top\vec{B}\vec{K}\right)$ is assumed to be a large value.

\begin{figure}[t]
\centering
\includegraphics[width=8cm]{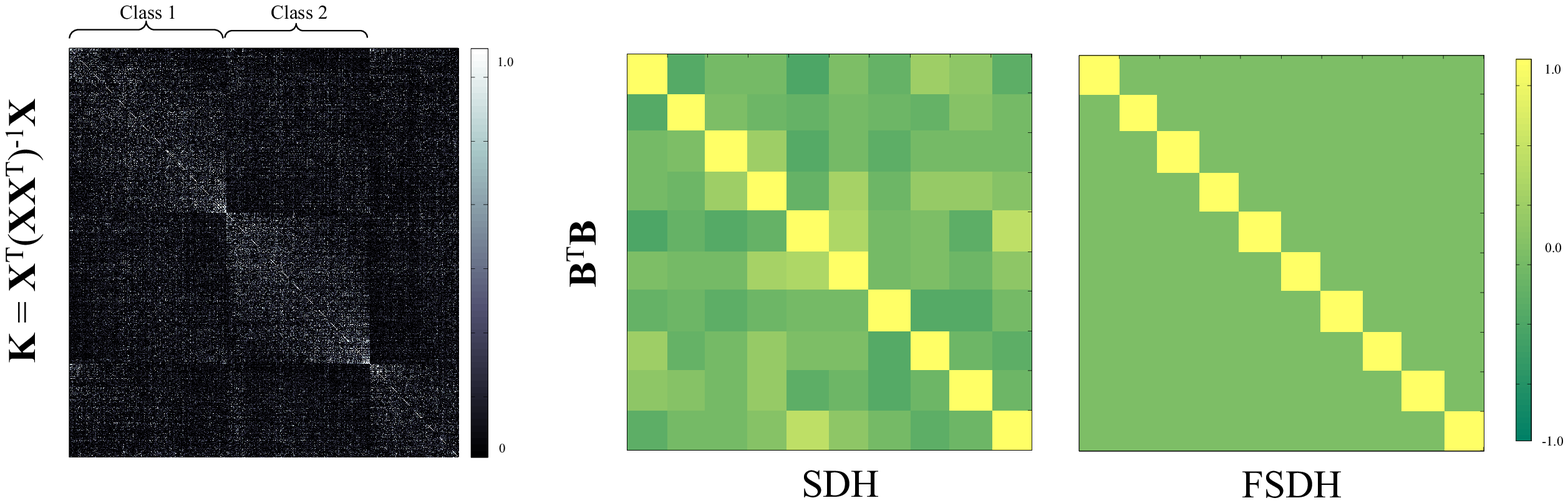}
\caption{
Visualization of matrix $\vec{K}$ (left) and matrix $\vec{B}^\top\vec{B}$ (right).
$\vec{B}^\top\vec{B}$ of SDH includes a 
``negative'' block in the non-diagonal components, 
and reduces $\mathrm{Tr}\left(\vec{K}\vec{B}^\top\vec{B}\vec{K}\right)$. 
}
\label{fig:K}
\end{figure}
\begin{figure*}[t]
\centering
\includegraphics[width=17cm]{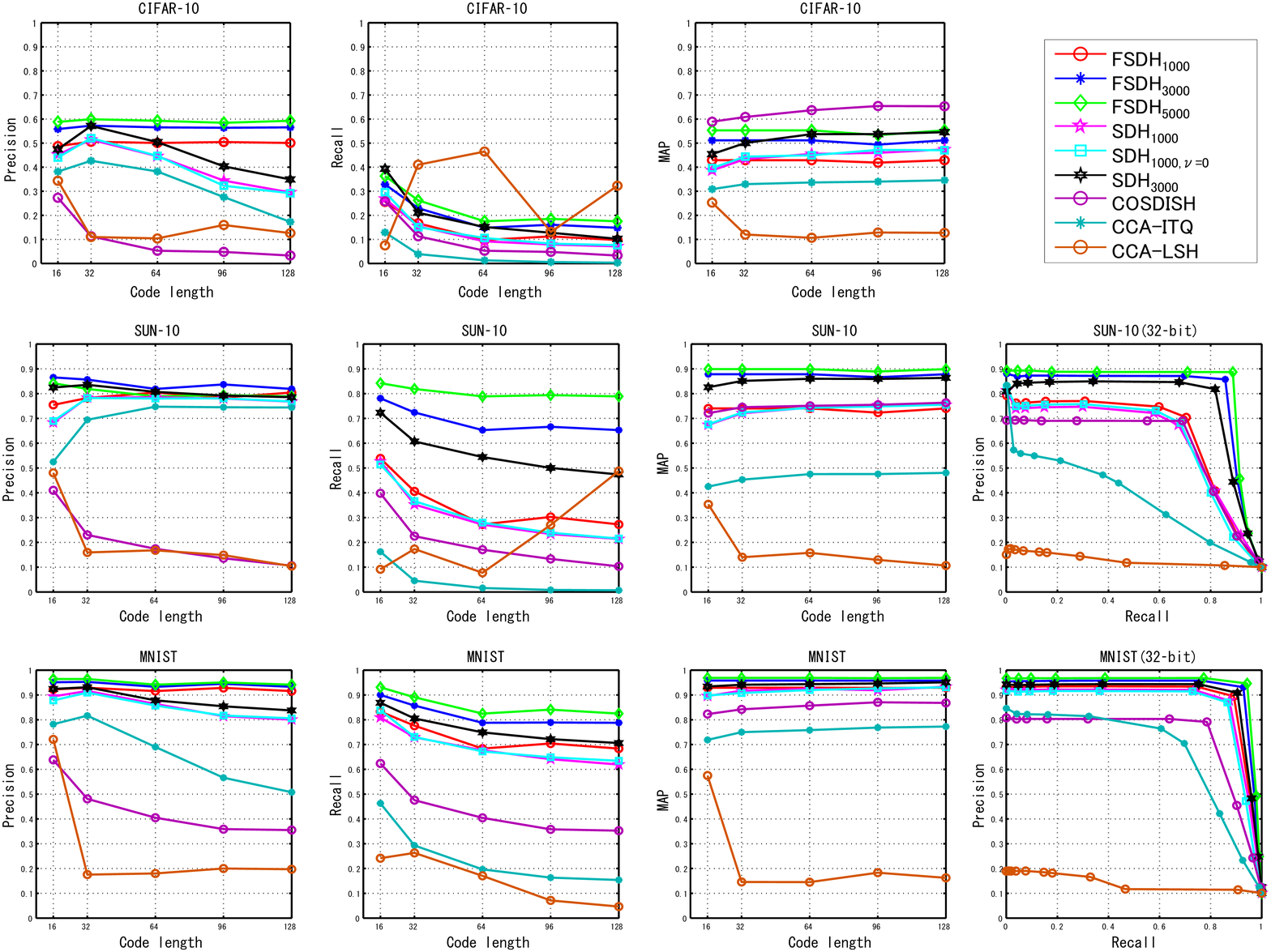}
\caption{Comparative results of precision, recall, and MAP for all datasets and all methods using code lengths $L=16, 32, 64, 96,$ and $128$. The proposed FSDH shows the best results and retains high precision and recall scores for longer code lengths.}
\label{fig:results}
\end{figure*}

Figure~\ref{fig:K} shows visualizations of matrices $\vec{K}$ and $\vec{B}^\top\vec{B}$ for SDH and FSDH. 
High-correlation areas of $\vec{K}$ are partitioned by each class block. 
$\vec{B}^\top\vec{B}$ of SDH includes a 
``negative'' block in the non-diagonal components, 
and reduces $\mathrm{Tr}\left(\vec{K}\vec{B}^\top\vec{B}\vec{K}\right)$. 
On the other hand, the proposed FSDH shows clear blocks; the diagonal blocks take the value $L$ and the non-diagonal blocks 0.

%
%
%

\section{Experiments} \label{sec:results}

%
%

\subsection{Datasets}

We tested the proposed method on three large-scale image datasets: 
CIFAR-10 \cite{CIFAR:TEC2009} \footnote{https://www.cs.toronto.edu/~kriz/cifar.html}, 
SUN-397 \cite{SUN:CVPR2010} \footnote{http://groups.csail.mit.edu/vision/SUN/}, 
and MNIST \cite{Lecun98:MNIST} \footnote{http://yann.lecun.com/exdb/mnist/}. 
The feature vectors of all datasets were normalized.
A multi-labeled NUS-WIDE dataset was not included due to the limitation that the proposed method can be applied only to single-label problems.


\underline{\textbf{CIFAR-10}} includes labeled subsets of 60,000 images. 
In this test, we used 512-dimensional GIST features \cite{Oliva:GIST2001} extracted from the images. $N \!=\! 59,000$ training samples and 1,000 test samples were used for evaluation. The number of classes was $C=10$, and included 
``airplane'', ``automobile'', ``bird'', $\ldots$, etc.

\underline{\textbf{SUN-397}} is a large-scale image dataset for scene recognition with 397 categories, and consists of 108,754 labeled images. We extracted 10 categories with $C=10$ and $N=5,000$ training samples. A total of 500 training samples per class and 1,000 test samples were used. We used 512-dimensional GIST features extracted from the images. Since we used $C=10$, we called the dataset 
``SUN-10'' in this study.

\underline{\textbf{MNIST}} includes an image dataset of handwritten digits. 
The feature vectors we used were given by $28 \times 28 = 784$ [pix] of data that were normalized. The number of classes was 
$C \!=\! 10$, \textit{i.e.}, $\text{`0'}\!\sim\!\text{`9'}$ digits. We used $N \!=\! 30,000$ training samples and 1,000 test samples for evaluation.

%
%

\subsection{Comparative methods and settings}

The proposed method was compared with four state-of-the-art supervised hashing methods: CCA-ITQ, CCA-LSH, SDH, and COSDISH~\cite{COSDISH:AAAI16}. Unsupervised or semi-supervised methods were not assessed. All methods were implemented in 
MATLAB R2012b and tested on an Intel i7-4770@3.4 GHz CPU with DDR3 SDRAM@32 GB.

\underline{\textbf{CCA-ITQ and LSH}}:
ITQ and LSH are state-of-the-art binary hashing methods.
They can be converted into supervised binary hashing methods by pre-processing feature vectors $\vec{X}$ using label information. Canonical correlation analysis (CCA) transformation was performed and feature vectors were normalized and set to zero mean. They generated the projection matrix $\vec{P}$, and binary codes were assigned by \eqref{eq:Px}.

\begin{table}
\centering
\caption{Comparison of performance with and without the bias term}
{\small
\begin{tabular}{lrr|rr|rr} \hline
& \multicolumn{2}{c}{CIFAR-10}   & \multicolumn{2}{c}{SUN-10} & \multicolumn{2}{c}{MNIST} \\ \hline
                   & MAP   & Pre. & MAP  & Pre.   & MAP     & Pre.\\ \hline
SDH($\nu=10^{-5}$) & 0.47 & 0.34 & 0.47 & 0.78  & 0.47   & 0.83 \\ 
SDH($\nu=0$) & 0.47 & 0.36 & 0.47 & 0.79  & 0.47   & 0.82 \\ \hline
\end{tabular}
\label{tbl:vali}
}
\end{table}
%

\begin{table}
\centering
\caption{$\|\vec{W}^\top\vec{Y}\|^2$ and $\nu \|\vec{P}^\top\vec{X}\|^2$ for all datasets for $L=64$.
$\|\vec{W}^\top\vec{Y}\|^2 \gg \nu \|\vec{P}^\top\vec{X}\|^2$ is shown.
}
\begin{tabular}{lrr} \hline
                & $\|\vec{W}^\top\vec{Y}\|^2$   & $\nu \|\vec{P}^\top\vec{X}\|^2$  \\ \hline
CIFAR-10 & 31.53 &0.0132 \\
SUN-10   &  9.16 &0.0045 \\
MNIST    & 22.96 &0.0124 \\ \hline
\end{tabular}
\label{tbl:norm}
\end{table}

\begin{table}
\centering
\caption{Computation times of learning samples for CIFAR-10 [s]}
{\footnotesize
\begin{tabular}{lrrrrr} \hline
$L$               &16   &32   &64   &96   &128 \\ \hline  
$\text{FSDH}_{1000}$        &5.05 &5.10 &5.62 &5.86 &5.78 \\
$\text{FSDH}_{3000}$        &45.28 &45.18 &45.09 &45.29 &45.22 \\
$\text{FSDH}_{5000}$       &121.53 &122.80 &121.86 &121.87 &124.77 \\
$\text{SDH}_{1000}$         &37.42 &55.04 &112.20 &185.79 &285.53 \\
$\text{SDH}_{1000}, \nu=0$ &27.50 &41.87 &60.56 &148.67 &199.96 \\
$\text{SDH}_{3000}$         &344.55 &343.38 &378.19 &474.44 &607.49 \\
COSDISH           &11.76 &41.42 &155.66 &349.82 &656.55 \\
CCA-ITQ           &1.29 &2.73 &5.61 &10.12 &14.25 \\
CCA-LSH           &0.00 &0.00 &0.00 &0.00 &0.01 \\ \hline
\end{tabular}
}
\label{tbl:comp}
\end{table}
\begin{table*}
\centering
\caption{Bit-scalability of FSDH and SDH for CIFAR-10 and 10,000 training samples}
\begin{tabular}{ccrrrrrr} \hline
       & $L$                &   32 &    64 &   128 &      256 &     512 &    1024 \\ \hline
FSDH   & learning time [s]  & 0.64 &  0.70 &  0.85 &     0.98 &    1.16 &    1.48 \\
       & Precision          & 0.50 &  0.47 &  0.47 &     0.47 &    0.47 &    0.47 \\
       & MAP                & 0.44 &  0.44 &  0.44 &     0.44 &    0.44 &    0.44 \\ \hline
SDH    & learning time [s]  & 6.38 & 14.92 & 47.42 &   284.00 & 1189.49 & 5230.43 \\
       & Precision          & 0.49 &  0.40 &  0.20 &     0.11 &    0.03 &    0.01 \\
       & MAP                & 0.44 &  0.47 &  0.48 &     0.48 &    0.46 &    0.44 \\ \hline                      
\end{tabular}
\label{tbl:bit}
\end{table*}
%

\underline{\textbf{COSDISH}} is a recently proposed supervised hashing method. COSDISH generates the projection matrix $\vec{P}$, as does ITQ. The feature vectors are transformed so they have zero mean and normalized through variance in pre-processing. We used open-source MATLAB code \footnote{http://cs.nju.edu.cn/lwj/}.

\underline{\textbf{SDH}} is a state-of-the-art supervised hashing method. We used $\lambda\! =\! 1$ and $\nu \!=\! 10^{-5}$ with the maximum number of iterations set to 5, anchor points $M \!=\! 1,000$ ($\text{SDH}_{1000}$) and $M \!=\! 3,000$ ($\text{SDH}_{3000}$), and kernel parameter $\sigma \!=\! 0.4$ for all datasets.
SDH generated the projection matrix $\vec{P}$, and binary codes were assigned by re-projection \eqref{eq:Px}. Furthermore, to show the validity of the FSDH approximation, we evaluated the case where $\nu=0$ ($\text{SDH}_{1000, \nu\!=\!0}$). We used open-source MATLAB code \footnote{https://github.com/bd622/DiscretHashing}.

\underline{\textbf{FSDH}}: The proposed method used the same parameters as SDH: anchor points $M\!=\!1,000$ ($\text{FSDH}_{1000}$) and $M\!=\!3,000$ ($\text{FSDH}_{3000}$), and kernel parameter $\sigma\!=\!0.4$ for all datasets. Moreover, we used $M\!=\!5,000$ ($\text{FSDH}_{5000}$). FSDH generated the projection matrix $\vec{P}$ and assigned binary codes through re-projection \eqref{eq:Px}, as in SDH. Our code will be made available to the public \footnote{https://github.com/goukoutaki/FSDH}, and is shown in Fig.~\ref{fig:code}.

%
%

\subsection{Results and discussion}

\begin{table}
\centering
\caption{MAP comparison of SUN-397($C=397$) 
}
{
\begin{tabular}{lrr} \hline
 Methods     &   MAP & Learning time[s] \\ \hline
 $\text{SDH}_{1000}$   & 0.025 & 17519 \\ 
 $\text{SDH}_{10000}$  & 0.113 & 71883 \\
$\text{FSDH}_{1000}$   & 0.030 &     7 \\ 
$\text{FSDH}_{10000}$  & 0.264 &   721 \\
$\text{FSDH}_{20000}$  & 0.442 &  3542 \\ \hline
FSH          \cite{CVPR14Lin}   & 0.142 & 29624 \\ 
LSVM-b       \cite{ACMMM14SVM}  & 0.042 & - \\ 
Top-RSBC+SGD \cite{Song_2015_ICCV} & 0.344 &  4663 \\ \hline
\end{tabular}
}
\label{tbl:SUN397}
\end{table}

Precision and recall were computed by calculating the Hamming 
distance between the training samples and the test samples 
with a Hamming radius of 2. Figure~\ref{fig:results} shows the results, 
in terms of precision, recall, and the mean average of precision (MAP), 
of the Hamming ranking for all methods and the three datasets. 
Code lengths of $L=16,32,64,96$, and $128$ were evaluated.

\underline{\textbf{CIFAR-10}}:
COSDISH shows the best MAP.
$\text{FSDH}_{5000}$ yielded the best precision and recall. 
Although COSDISH showed a satisfactory MAP, the precision was low. 
In SDH and FSDH, increasing the number of anchor points improved the performance. 
As the code length increases, SDH reduces precision. However, FSDH maintains high precision and recall. 
This is a significant advantage of the proposed method. 
In general, by increasing the code length, precision tends to decrease 
with such a narrow threshold of a Hamming radius of 2.

\underline{\textbf{SUN-10}}:
In this dataset, the results for FSDH were significantly better. In particular, the recall rates of FSDH remained high in spite of long code lengths.
When the SDH and FSDH had the same number of anchor points, FSDH was clearly superior. The MAP of COSDISH was comparable to that of SDH; however, the precision and recall of COSDISH were not as good as those of CIFAR-10.

\underline{\textbf{MNIST}}:
FSDH yielded the best results in all datasets with the same trends. It retained high precision and recall even with large values of code length. 

The graphs on the right of Fig.~\ref{fig:results} 
show the precision-recall ROC curves based on Hamming ranking. 
FSDH shows better performance than SDH with the same number of anchor points.
In particular, the SUN dataset yielded distinct results compared 
with the other methods.

%
%

\subsubsection{Validation of FSDH approximation}

Table~\ref{tbl:vali} shows the comparative results of 
$\text{SDH}_{1000}$ with $\nu \!=\! 10^{-5}$ and $\text{SDH}_{1000}$ with $\nu \!=\! 0$. 
For all datasets, the results of $\text{SDH}_{1000}$ and $\text{SDH}_{1000}$ with $\nu \!=\! 0$, 
were almost identical. 
Table~\ref{tbl:norm} shows $\|\vec{W}^\top\vec{Y}\|^2$ 
and $\nu \|\vec{P}^\top\vec{X}\|^2$ of $\text{SDH}_{1000}$ for all datasets with 
$L=64$ and after optimization.
We can confirm $\|\vec{W}^\top\vec{Y}\|^2 \gg \nu \|\vec{P}^\top\vec{X}\|^2$.
This means that the FSDH approximation was appropriate for supervised hashing.

%
%

%

%
%

\subsubsection{Computation time}

Table~\ref{tbl:comp} shows the computation time of each method for CIFAR-10. The time for $\text{FSDH}_{1000}$ was almost identical to that of CCA=ITQ. As the number of anchors increased, the computational time increased for SDH and FSDH. The computational time for SDH and COSDISH increased with the code length. The number of iterations of the DCC method depended on the code length.

%
%

\subsubsection{Bit scalability and larger classes}

Table~\ref{tbl:bit} shows the comparative results in terms of computational time and performance 
with a wide range of code lengths $L=32 \!\sim\! 1024$ for the CIFAR-10 dataset. $N\!=\!10,000$ training samples, 
1,000 test samples, and 1,000 anchors were used. The computation time of FSDH was almost identical 
in terms of code length because the main computation in FSDH involved matrix multiplication and inversion 
$\left(\vec{X}\vec{X}^\top\right)^{-1}$ of \eqref{eq:P}. In practice, the inverse matrix was 
not computed directly, and Cholesky decomposition was performed. On the contrary, the computation time for SDH exponentially increased and precision decreased significantly. This means that the DCC method fell into local 
minima in cases of large code length.
%

In general, large bits are useful for a large number of classes. 
Table~\ref{tbl:SUN397} shows the results of larger classes of the SUN dataset.
$L\!=\!512$-bits, $C\!=\!397$ classes and $N\!=\!79,400$ training samples are used.
FSDH achieves high precision, high MAP and lower computational time compared with SDH.
When $M\!=\!20,000$ is used, $\text{MAP} \!=\! 0.442$ can be ontained by FSDH.
Here, SDH was not able to finish after three days of computation 
in our computational environment.
In the experiments, we found that a large number of anchor points can improve performance.
However it requres more computation.
Therefore FSDH can use a large number of anchor points in a realistic computation time compared with SDH.
For reference, 
we refer to the results of \textit{fast supervised hashing} (FSH), LSVM-b and Top-RSBC+SGD which are
reprinted from \cite{CVPR14Lin,ACMMM14SVM,Song_2015_ICCV}.
Although FSDH outperforms those methods, note that those methods use different computational environments, 
feature vectors and code lengths.

%
%
%

\section{Conclusion}

In this paper, we simplified the SDH model to an FSDH model by approximating the bias term, and provided exact solutions for the proposed FDSH model. The FSDH approximation was validated by comparative experiments with the SDH model. FSDH is easy to implement and outperformed several state-of-the-art supervised hashing methods. In particular, in the case of large code lengths, FSDH can maintain performance without losing precision. In future work, we intend to use this idea for other hashing models.

%
%
%

\appendix

\section*{Appendix}
\section{Proof of Lemma 4.2} \label{sec:appendix_Lemma}

The optimization problem in (15) is known as the \textit{resource allocation problem} \cite{Bellman1962,Dreyfus1977,Ibaraki1988}.
Here we present a simple proof for the solution.

The constraint $\sum_{i=1}^N x_i = L$ can be regarded as a surface equation in an N-dimensional space $(x_1,x_2,\ldots,x_N)$.
On the other hand, the gradient vector of the object function $\sum_{i=1}^N f(x_i)$ is defined as
\begin{equation}
\mathbf{g} := [f'(x_1),f'(x_2),\ldots,f'(x_N)]^\top
\label{eq:grad_f}
\end{equation}
where $f'(\cdot)$ is the differentiated version of $f(\cdot)$ and the $i$-th element (the gradient in the $i$-th direction) is given by 
$\tfrac{\partial}{\partial x_i} \sum_{j} f(x_j) =  \sum_{j} \tfrac{\partial x_j}{\partial x_i}\tfrac{\partial}{\partial x_j}f(x_j)$, and $\tfrac{\partial x_j}{\partial x_i}$ becomes 1 if $i = j$ or 0 if $i \neq j$.

Then, the gradient along the surface is obtained as the projection of $\mathbf{g}$ onto the surface, and computed as the inner-product of $\mathbf{g}$ and a set of vectors $\{\mathbf{n}^\perp\}$ perpendicular to the normal vector of the surface:
\begin{equation}
\mathbf{n}:=\tfrac{1}{\sqrt{N}}[1,1,\ldots,1]^\top\in\mathbb{R}^N,
\label{eq:normal_vec}
\end{equation}
and the projected gradient $\mathbf{g}^\top \mathbf{n}^\perp$ becomes 0 at the global extermum point on the surface.
This also indicates $\mathbf{g}$ and $\mathbf{n}$ are completely parallel and their inner-product becomes
\begin{equation}
\Bigl( \frac{\mathbf{g}}{\|\mathbf{g}\|_2} \Bigr)^\top \mathbf{n} = 1 \quad \Rightarrow\quad \mathbf{g}^\top \mathbf{n} = \|\mathbf{g}\|_2.
\label{eq:inner_prod}
\end{equation}
Substituting Eqs. \eqref{eq:grad_f} and \eqref{eq:normal_vec} into \eqref{eq:inner_prod}, we get
\begin{equation}
\tfrac{1}{\sqrt{N}} \sum_i f'(x_i) = \sqrt{ \sum_i f'(x_i)^2 }
\end{equation}
Additionally, when we express $f'(x_i)$ as $\sqrt{ f'(x_i)^2 }$ and $\tfrac{1}{\sqrt{N}}$ as $\tfrac{1}{N}\sqrt{N}$, we get
\begin{equation}
\tfrac{1}{N} \sum_i \sqrt{ f'(x_i)^2 } = \sqrt{ \tfrac{1}{N} \sum_i f'(x_i)^2 }.
\end{equation}
The shape of this equality actually corresponds to Jensen's inequality: $\sum_i p_i h(y_i) \ge h(\sum_i p_i y_i)$ where $\sum_i p_i = 1$, and the equality holds if and only if $\{y_i\}$ i.e. $\{f'(x_i)^2\}$ are all equal: 
\begin{equation}
f'(x_1)^2=f'(x_2)^2=\ldots=f'(x_N)^2 \\
\end{equation}
Additionally, when $f(\cdot)$ is a convex function, $f'(\cdot)$ becomes an injective function because $f''(\cdot) \ge 0$ is a monotonically increasing function.
Also, if the sign of $f'(\cdot)$ does not change within the valid range of $x_i$ (the case considered in this paper), $f'(\cdot)^2$ becomes injective.
Hence, 
\begin{equation}
x_1=x_2=\ldots=x_N.
\label{eq:all_same}
\end{equation}
Finally, substituting this \eqref{eq:all_same} into the condition $\sum_{i=1}^N x_i = L$, we get
\begin{equation}
\forall_i \ x_i = \frac{L}{N}.
\end{equation}
$\square$

%
%
%

\section{Complete data of the experimental results}
\subsection{Recall, precision and MAP}
Tables \ref{tbl:cifar-10}$\sim$\ref{tbl:MNIST} show
recall, precision and MAP for all datasets when $L=16,32,64,96$ and $128$.
These were computed by calculating the Hamming 
distance between the training samples and the test samples 
with a Hamming radius of 2. 

\begin{table*}
\centering
\caption{CIFAR-10, GIST-512, training samples $N=59,000$ and test samples $1,000$}
{\footnotesize
\begin{tabular}{l|rrrrr|rrrrr|rrrrr} \hline
    & \multicolumn{5}{c}{Precision} & \multicolumn{5}{c}{Recall} & \multicolumn{5}{c}{MAP} \\ \hline
$L$	&	16 	&	32 	&	64 	&	96 	&	128 	&	16 	&	32 	&	64 	&	96 	&	128 	&	16 	&	32 	&	64 	&	96 	&	128 	\\ \hline
$\text{FSDH}_{1000}$	&	0.488 	&	0.506 	&	0.501 	&	0.505 	&	0.501 	&	0.256 	&	0.167 	&	0.097 	&	0.112 	&	0.097 	&	0.429 	&	0.429 	&	0.429 	&	0.419 	&	0.429 	\\
$\text{FSDH}_{3000}$	&	0.559 	&	0.573 	&	0.566 	&	0.564 	&	0.566 	&	0.328 	&	0.230 	&	0.148 	&	0.161 	&	0.148 	&	0.512 	&	0.512 	&	0.512 	&	0.494 	&	0.512 	\\
$\text{FSDH}_{5000}$	&	{\bf 0.589} 	&	{\bf 0.599} 	&	{\bf 0.593} 	&	{\bf 0.585} 	&	{\bf 0.593} 	&	{\bf 0.363} 	&	{\bf 0.262} 	&	{\bf 0.175} 	&	{\bf 0.185} 	&	{\bf 0.175} 	&	0.553 	&	0.553 	&	0.553 	&	0.533 	&	0.553 	\\ \hline
$\text{SDH}_{1000}$	&	0.456 	&	0.517 	&	0.427 	&	0.339 	&	0.276 	&	0.306 	&	0.147 	&	0.095 	&	0.077 	&	0.068 	&	0.409 	&	0.436 	&	0.457 	&	0.464 	&	0.470 	\\
$\text{SDH}_{1000,\nu=0}$	&	0.447 	&	0.511 	&	0.453 	&	0.346 	&	0.292 	&	0.298 	&	0.154 	&	0.099 	&	0.082 	&	0.074 	&	0.399 	&	0.440 	&	0.445 	&	0.459 	&	0.468 	\\
$\text{SDH}_{3000}$	&	0.511 	&	0.584 	&	0.483 	&	0.403 	&	0.346 	&	0.354 	&	0.195 	&	0.147 	&	0.121 	&	0.106 	&	0.471 	&	0.520 	&	0.529 	&	0.542 	&	0.548 	\\ \hline
COSDISH	&	0.262 	&	0.120 	&	0.061 	&	0.046 	&	0.031 	&	0.251 	&	0.118 	&	0.061 	&	0.046 	&	0.031 	&	{\bf 0.574} 	&	{\bf 0.615} 	&	{\bf 0.625} 	&	{\bf 0.644} 	&	{\bf 0.654} 	\\
CCA-ITQ	&	0.373 	&	0.427 	&	0.352 	&	0.267 	&	0.203 	&	0.139 	&	0.040 	&	0.014 	&	0.006 	&	0.004 	&	0.307 	&	0.329 	&	0.339 	&	0.341 	&	0.344 	\\
CCA-LSH	&	0.332 	&	0.159 	&	0.102 	&	0.146 	&	0.150 	&	0.056 	&	0.143 	&	0.487 	&	0.213 	&	0.204 	&	0.240 	&	0.141 	&	0.101 	&	0.125 	&	0.127 	\\ \hline
\end{tabular}
}
\label{tbl:cifar-10}
\vspace{0.5\baselineskip}
\caption{SUN-10, GIST-512, training samples $N=5,000$ and test samples $1,000$}
{\footnotesize
\begin{tabular}{l|rrrrr|rrrrr|rrrrr} \hline
    & \multicolumn{5}{c}{Precision} & \multicolumn{5}{c}{Recall} & \multicolumn{5}{c}{MAP} \\ \hline
$L$	&	16 	&	32 	&	64 	&	96 	&	128 	&	16 	&	32 	&	64 	&	96 	&	128 	&	16 	&	32 	&	64 	&	96 	&	128 	\\ \hline
$\text{FSDH}_{1000}$	&	0.754 	&	0.782 	&	0.804 	&	0.784 	&	0.804 	&	0.539 	&	0.406 	&	0.272 	&	0.303 	&	0.272 	&	0.740 	&	0.740 	&	0.740 	&	0.723 	&	0.740 	\\
$\text{FSDH}_{3000}$	&	{\bf 0.865} 	&	{\bf 0.857} 	&	{\bf 0.819} 	&	{\bf 0.837} 	&	{\bf 0.819} 	&	0.781 	&	0.724 	&	0.653 	&	0.666 	&	0.653 	&	0.878 	&	0.878 	&	0.878 	&	0.866 	&	0.878 	\\
$\text{FSDH}_{5000}$	&	0.842 	&	0.819 	&	0.789 	&	0.794 	&	0.789 	&	{\bf 0.842} 	&	{\bf 0.819} 	&	{\bf 0.789} 	&	{\bf 0.794} 	&	{\bf 0.789} 	&	{\bf 0.899} 	&	{\bf 0.899} 	&	{\bf 0.899} 	&	{\bf 0.889} 	&	{\bf 0.899} 	\\ \hline
$\text{SDH}_{1000}$	&	0.682 	&	0.788 	&	0.778 	&	0.774 	&	0.770 	&	0.527 	&	0.370 	&	0.263 	&	0.234 	&	0.209 	&	0.674 	&	0.722 	&	0.739 	&	0.748 	&	0.760 	\\
$\text{SDH}_{1000,\nu=0}$	&	0.690 	&	0.781 	&	0.780 	&	0.771 	&	0.770 	&	0.544 	&	0.360 	&	0.279 	&	0.229 	&	0.215 	&	0.684 	&	0.717 	&	0.749 	&	0.752 	&	0.756 	\\
$\text{SDH}_{3000}$	&	0.829 	&	0.831 	&	0.807 	&	0.791 	&	0.785 	&	0.733 	&	0.624 	&	0.540 	&	0.495 	&	0.480 	&	0.837 	&	0.840 	&	0.857 	&	0.859 	&	0.866 	\\ \hline
COSDISH	&	0.425 	&	0.229 	&	0.173 	&	0.141 	&	0.111 	&	0.406 	&	0.227 	&	0.170 	&	0.136 	&	0.108 	&	0.682 	&	0.724 	&	0.744 	&	0.760 	&	0.767 	\\
CCA-ITQ	&	0.524 	&	0.699 	&	0.744 	&	0.746 	&	0.744 	&	0.153 	&	0.044 	&	0.017 	&	0.012 	&	0.009 	&	0.411 	&	0.454 	&	0.472 	&	0.479 	&	0.484 	\\
CCA-LSH	&	0.486 	&	0.148 	&	0.161 	&	0.175 	&	0.171 	&	0.087 	&	0.271 	&	0.161 	&	0.061 	&	0.043 	&	0.345 	&	0.133 	&	0.133 	&	0.162 	&	0.144 	\\ \hline
\end{tabular}
}
\label{tbl:SUN-10}
\vspace{0.5\baselineskip}
\caption{MNIST, training samples $N=30,000$ and test samples $1,000$}
{\footnotesize
\begin{tabular}{l|rrrrr|rrrrr|rrrrr} \hline
    & \multicolumn{5}{c}{Precision} & \multicolumn{5}{c}{Recall} & \multicolumn{5}{c}{MAP} \\ \hline
$L$	&	16 	&	32 	&	64 	&	96 	&	128 	&	16 	&	32 	&	64 	&	96 	&	128 	&	16 	&	32 	&	64 	&	96 	&	128 	\\ \hline
$\text{FSDH}_{1000}$	&	0.922 	&	0.929 	&	0.916 	&	0.928 	&	0.916 	&	0.833 	&	0.776 	&	0.684 	&	0.704 	&	0.684 	&	0.929 	&	0.929 	&	0.929 	&	0.928 	&	0.929 	\\
$\text{FSDH}_{3000}$	&	0.951 	&	0.953 	&	0.933 	&	0.945 	&	0.933 	&	0.900 	&	0.857 	&	0.788 	&	0.789 	&	0.788 	&	0.958 	&	0.958 	&	0.958 	&	0.958 	&	0.958 	\\
$\text{FSDH}_{5000}$	&	{\bf 0.964} 	&	{\bf 0.964} 	&	{\bf 0.942} 	&	{\bf 0.950} 	&	{\bf 0.942} 	&	{\bf 0.931} 	&	{\bf 0.891} 	&	{\bf 0.825} 	&	{\bf 0.841} 	&	{\bf 0.825} 	&	{\bf 0.969} 	&	{\bf 0.969} 	&	{\bf 0.969} 	&	{\bf 0.968} 	&	{\bf 0.969} 	\\ \hline
$\text{SDH}_{1000}$	&	0.896 	&	0.916 	&	0.862 	&	0.825 	&	0.802 	&	0.809 	&	0.736 	&	0.677 	&	0.645 	&	0.627 	&	0.898 	&	0.921 	&	0.926 	&	0.926 	&	0.930 	\\
$\text{SDH}_{1000,\nu=0}$	&	0.888 	&	0.906 	&	0.878 	&	0.822 	&	0.809 	&	0.840 	&	0.736 	&	0.686 	&	0.651 	&	0.636 	&	0.898 	&	0.911 	&	0.920 	&	0.930 	&	0.933 	\\
$\text{SDH}_{3000}$	&	0.909 	&	0.930 	&	0.889 	&	0.853 	&	0.843 	&	0.839 	&	0.798 	&	0.745 	&	0.714 	&	0.705 	&	0.911 	&	0.937 	&	0.947 	&	0.947 	&	0.948 	\\ \hline
COSDISH	&	0.640 	&	0.488 	&	0.395 	&	0.376 	&	0.344 	&	0.626 	&	0.488 	&	0.395 	&	0.375 	&	0.343 	&	0.818 	&	0.844 	&	0.860 	&	0.865 	&	0.863 	\\
CCA-ITQ	&	0.782 	&	0.824 	&	0.686 	&	0.574 	&	0.491 	&	0.449 	&	0.294 	&	0.202 	&	0.165 	&	0.151 	&	0.710 	&	0.745 	&	0.761 	&	0.766 	&	0.773 	\\
CCA-LSH	&	0.723 	&	0.142 	&	0.190 	&	0.197 	&	0.197 	&	0.205 	&	0.349 	&	0.146 	&	0.089 	&	0.068 	&	0.562 	&	0.124 	&	0.148 	&	0.155 	&	0.152 	\\ \hline
\end{tabular}
}
\label{tbl:MNIST}
\end{table*}

\subsection{ROC curves}
Figure \ref{fig:results} shows precision-recall ROC curves based on
Hamming distance ranking for all datasets when $L=16 \sim 128$.

\begin{figure*}[t]
\centering
\includegraphics[width=16cm]{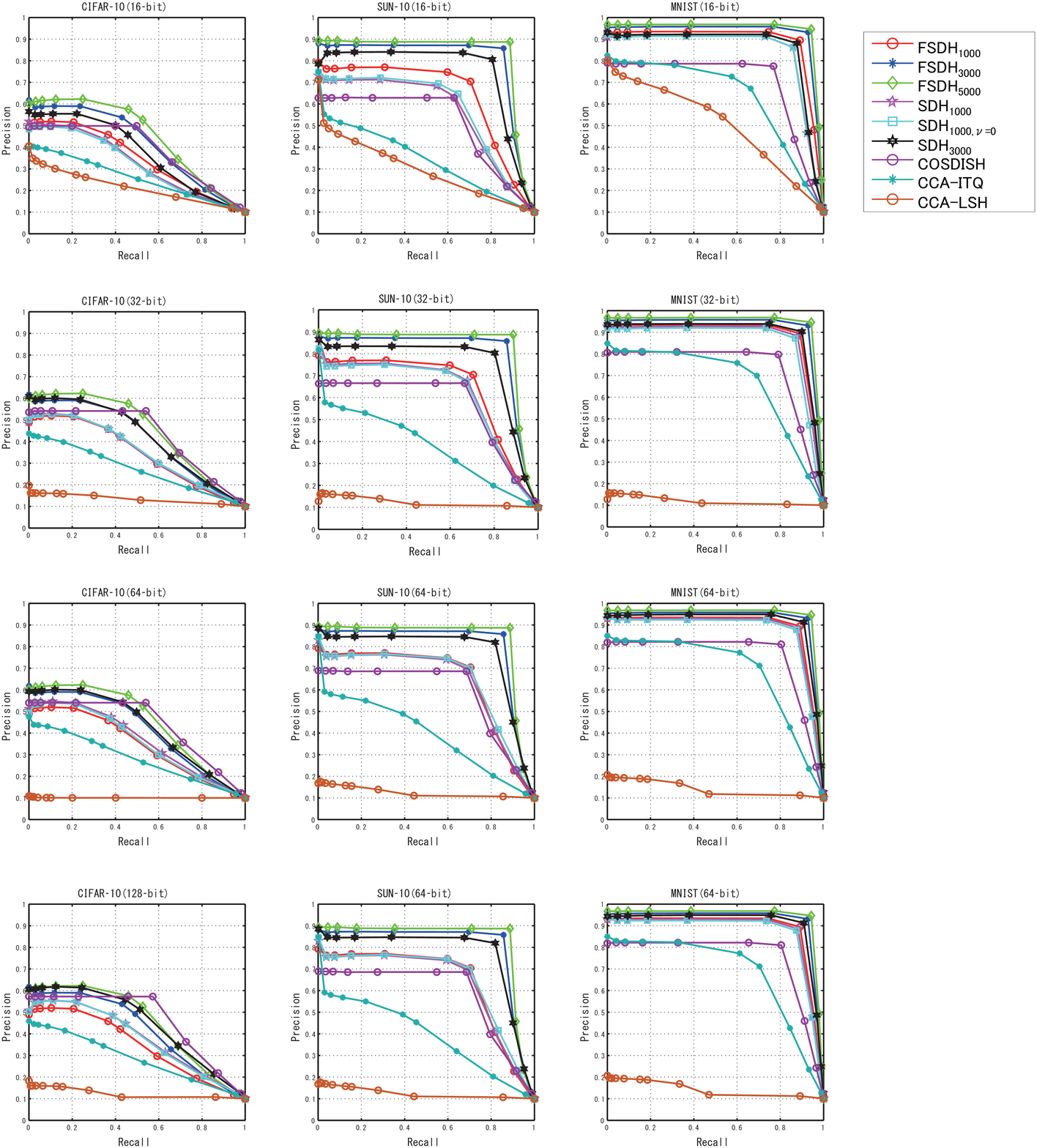}
\caption{
Comparative results of ROC for all datasets and all methods.}
\label{fig:results}
\end{figure*}


\section{Loss comparison}
We define $\vec{W}$\textit{-loss} and $\vec{P}$\textit{-loss} of the SDH model in (4) as follows:
\begin{equation}
\begin{aligned}
\vec{W}\textit{-loss} & = \|\vec{Y}-\vec{W}^\top\vec{B}\|^2, \\
\vec{P}\textit{-loss} & = \|\vec{B}-\vec{P}^\top\vec{X}\|^2.
\end{aligned}
\end{equation}
Tables \ref{tbl:loss-cifar-10} $\sim$ \ref{tbl:loss-mnist} show
each loss of SDH and FSDH after optimization for the CIFAR-10, SUN-10 and MNIST datasets.
As described in sec.4.1, FSDH can minimize $\vec{W}$-{\it loss} exactly.
Therefore, for all datasets, FSDH results in a lower value of $\vec{W}$-{\it loss} than SDH.
Furthermore, as described in sec.4.3,
FSDH can also reduce $\vec{P}$-{\it loss}.
For CIFAR-10, SDH results in a lower value of $\vec{P}$-{\it loss} than FSDH.
For SUN-10 and MNIST, FSDH results in a lower value of $\vec{P}$-{\it loss} than SDH.

\begin{table}
\centering
\caption{{\it loss} of CIFAR-10}
\begin{tabular}{r|rr|rr} \hline
$L$ & SDH & & FSDH & \\
 & 
 $\vec{W}$-{\it loss} & $\vec{P}$-{\it loss} &
 $\vec{W}$-{\it loss} & $\vec{P}$-{\it loss} \\ \hline
16 & 0.0044 &729.0250 &0.0026 &716.5170 \\
32 & 0.0017 &998.6710 &0.0013 &1013.3100 \\
64 & 0.0008 &1425.4700 &0.0006 &1433.0300 \\
128 & 0.0006 &1980.5200 &0.0003 &2026.6200 \\ \hline
\end{tabular}
\label{tbl:loss-cifar-10}
\vspace{0.5\baselineskip}
\end{table}

\begin{table}
\centering
\caption{{\it loss} of SUN-10}
\begin{tabular}{r|rr|rr} \hline
$L$ & SDH & & FSDH & \\
 & 
 $\vec{W}$-{\it loss} & $\vec{P}$-{\it loss} &
 $\vec{W}$-{\it loss} & $\vec{P}$-{\it loss} \\ \hline
16&0.0124 &160.7050 &0.0088 &151.6680 \\
32&0.0054 &224.2900 &0.0044 &214.4910 \\
64&0.0028 &316.3380 &0.0022 &303.3370 \\
128&0.0068 &448.1170 &0.0011 &428.9830 \\
\hline
\end{tabular}
\label{tbl:loss-sun-10}
\end{table}
\begin{table}
\centering
\caption{{\it loss} of MNIST}
\begin{tabular}{r|rr|rr} \hline
$L$ & SDH & & FSDH & \\
 & 
 $\vec{W}$-{\it loss} & $\vec{P}$-{\it loss} &
 $\vec{W}$-{\it loss} & $\vec{P}$-{\it loss} \\ \hline
16&0.0058 &271.8340 &0.0036 &268.3930 \\
32&0.0037 &386.6100 &0.0018 &379.5650 \\
64&0.0011 &569.8690 &0.0009 &536.7860 \\
128&0.0008 &784.7970 &0.0005 &759.1300  \\
\hline
\end{tabular}
\label{tbl:loss-mnist}
\end{table}

\bibliographystyle{ieee}
\bibliography{refs_very_tiny}

\begin{thebibliography}{10}\itemsep=-1pt

\bibitem{Bellman1962}
R.~E. Bellman and S.~E. Dreyfus.
\newblock {\em Applied dynamic programming}.
\newblock Princeton University Press, 1962.

\bibitem{ACMMM14SVM}
F.~Cakir and S.~Sclaroff.
\newblock Supervised hashing with error correcting codes.
\newblock In {\em ACM MM}, pages 785--788, 2014.

\bibitem{Calonder:2010:BBR}
M.~Calonder, V.~Lepetit, C.~Strecha, and P.~Fua.
\newblock {BRIEF}: binary robust independent elementary features.
\newblock In {\em ECCV}, pages 778--792, 2010.

\bibitem{NFM:NPIS2005}
I.~S. Dhillon and S.~Sra.
\newblock Generalizednnonnegative matrix approximations with {Bregman}
  divergences.
\newblock In {\em NIPS}, pages 283--290, 2005.

\bibitem{Dreyfus1977}
S.~E. Dreyfus and A.~M. Law.
\newblock {\em The art and theory of dynamic programming}.
\newblock Academic Press, 1977.

\bibitem{Gionis:1999:LSH}
A.~Gionis, P.~Indyk, and R.~Motwani.
\newblock Similarity search in high dimensions via hashing.
\newblock In {\em Int. Conf. Very Large Data Bases (VLDB)}, pages 518--529,
  1999.

\bibitem{Gog:2016:FCH}
S.~Gog and R.~Venturini.
\newblock Fast and compact hamming distance index.
\newblock In {\em Int. ACM SIGIR Conf. Research \& Devel. Info. Retriev.},
  pages 285--294, 2016.

\bibitem{Golub:1996}
G.~H. Golub and C.~F. Van~Loan.
\newblock {\em Matrix computations (3rd ed.)}.
\newblock Johns Hopkins Univ. Press, 1996.

\bibitem{Gong:2012_ITQ_TPAMI}
Y.~Gong, S.~Lazebnik, A.~Gordo, and F.~Perronnin.
\newblock Iterative quantization: a procrustean approach to learning binary
  codes for large-scale image retrieval.
\newblock {\em IEEE T. PAMI}, 35(12):2916--2929, 2013.

\bibitem{Hadamard:1893}
J.~Hadamard.
\newblock R\'{e}solution \'{d}une question relative aux d\'{e}terminants.
\newblock {\em Bulletin Sci. Math.}, 17:240--246, 1893.

\bibitem{HEO:SH_TPAMI2015}
J.~P. Heo, Y.~Lee, J.~He, S.~F. Chang, and S.~E. Yoon.
\newblock Spherical hashing: Binary code embedding with hyperspheres.
\newblock {\em IEEE T. PAMI}, 37(11):2304--2316, 2015.

\bibitem{CCA}
H.~Hotelling.
\newblock Relations between two sets of variables.
\newblock {\em Biometrika}, pages 312--377, 1936.

\bibitem{Ibaraki1988}
T.~Ibaraki and N.~Katoh.
\newblock {\em Resource allocation problems: algorithms approaches}.
\newblock The MIT Press, 1988.

\bibitem{COSDISH:AAAI16}
W.~C. Kang, W.~J. Li, and Z.~H. Zhou.
\newblock Column sampling based discrete supervised hashing.
\newblock In {\em AAAI}, pages 1230--1236, 2016.

\bibitem{Koutaki:2016}
G.~Koutaki.
\newblock Binary continuous image decomposition for multi-view display.
\newblock {\em ACM TOG}, 35(4):69:1--69:12, 2016.

\bibitem{CIFAR:TEC2009}
A.~Krizhevsky.
\newblock Learning multiple layers of features from tiny images.
\newblock Technical report, Univ. Toronto, 2009.

\bibitem{BRE_NIPS2009}
B.~Kulis and T.~Darrell.
\newblock Learning to hash with binary reconstructive embeddings.
\newblock In {\em NIPS}, pages 1042--1050, 2009.

\bibitem{Lecun98:MNIST}
Y.~Lecun, L.~Bottou, Y.~Bengio, and P.~Haffner.
\newblock Gradient-based learning applied to document recognition.
\newblock In {\em Proc. IEEE}, pages 2278--2324, 1998.

\bibitem{HashICML13a}
X.~Li, G.~Lin, C.~Shen, A.~{van den Hengel}, and A.~Dick.
\newblock Learning hash functions using column generation.
\newblock In {\em ICML}, 2013.

\bibitem{CVPR14Lin}
G.~Lin, C.~Shen, Q.~Shi, A.~{van den Hengel}, and D.~Suter.
\newblock Fast supervised hashing with decision trees for high-dimensional
  data.
\newblock In {\em CVPR}, pages 1971--1978, 2014.

\bibitem{HGraph:NIPS2014}
W.~Liu, C.~Mu, S.~Kumar, and S.-F. Chang.
\newblock Discrete graph hashing.
\newblock In {\em NIPS}, pages 3419--3427, 2014.

\bibitem{HGraph:ICML11}
W.~Liu, J.~Wang, and S.~fu~Chang.
\newblock Hashing with graphs.
\newblock In {\em ICML}, 2011.

\bibitem{CVPR12:KSH}
W.~Liu, J.~Wang, R.~Ji, Y.-G. Jiang, and S.-F. Chang.
\newblock Supervised hashing with kernels.
\newblock In {\em CVPR}, pages 2074--2081, 2012.

\bibitem{Nguyen:2014:SDH}
V.~A. Nguyen, J.~Lu, and M.~N. Do.
\newblock Supervised discriminative hashing for compact binary codes.
\newblock In {\em ACM MM}, pages 989--992, 2014.

\bibitem{Oliva:GIST2001}
A.~Oliva and A.~Torralba.
\newblock Modeling the shape of the scene: a holistic representation of the
  spatial envelope.
\newblock {\em IJCV}, 42(3):145--175, 2001.

\bibitem{Schrijver:1986}
A.~Schrijver.
\newblock {\em Theory of linear and integer programming}.
\newblock John Wiley \& Sons, Inc., 1986.

\bibitem{Shen_2015_ICCV}
F.~Shen, W.~Liu, S.~Zhang, Y.~Yang, and H.~Tao~Shen.
\newblock Learning binary codes for maximum inner product search.
\newblock In {\em ICCV}, pages 4148--4156, 2015.

\bibitem{Shen_2015_CVPR}
F.~Shen, C.~Shen, W.~Liu, and H.~Tao~Shen.
\newblock Supervised discrete hashing.
\newblock In {\em CVPR}, pages 37--45, 2015.

\bibitem{TIP2016binary}
F.~Shen, X.~Zhou, Y.~Yang, J.~Song, H.~T. Shen, and D.~Tao.
\newblock A fast optimization method for general binary code learning.
\newblock {\em IEEE T. Image Process. (TIP)}, 25(12):5610--5621, 2016.

\bibitem{ECCV15:KSDH}
X.~Shi, F.~Xing, J.~Cai, Z.~Zhang, Y.~Xie, and L.~Yang.
\newblock Kernel-based supervised discrete hashing for image retrieval.
\newblock In {\em ECCV}, pages 419--433, 2016.

\bibitem{SlawskiHL13}
M.~Slawski, M.~Hein, and P.~Lutsik.
\newblock Matrix factorization with binary components.
\newblock In {\em NIPS}, pages 3210--3218, 2013.

\bibitem{Song_2015_ICCV}
D.~Song, W.~Liu, R.~Ji, D.~A. Meyer, and J.~R. Smith.
\newblock Top rank supervised binary coding for visual search.
\newblock In {\em ICCV}, pages 1922--1930, 2015.

\bibitem{Sylvester:1867}
J.~Sylvester.
\newblock Thoughts on inverse orthogonal matrices, simultaneous sign
  successions, and tessellated pavements in two or more colours, with
  applications to {N}ewton's rule, ornamental tile-work, and the theory of
  numbers.
\newblock {\em Philos. Magazine}, 34:461--475, 1867.

\bibitem{Wang_2016_CVPR}
X.~Wang, T.~Zhang, G.-J. Qi, J.~Tang, and J.~Wang.
\newblock Supervised quantization for similarity search.
\newblock In {\em CVPR}, pages 2018--2026, 2016.

\bibitem{Weiss:SpectralH_NIPS2008}
Y.~Weiss, A.~Torralba, and R.~Fergus.
\newblock Spectral hashing.
\newblock In {\em NIPS}, pages 1753--1760, 2009.

\bibitem{SUN:CVPR2010}
J.~Xiao, J.~Hays, K.~A. Ehinger, A.~Oliva, and A.~Torralba.
\newblock {SUN} database: large-scale scene recognition from abbey to zoo.
\newblock In {\em CVPR}, pages 3485--3492, 2010.

\end{thebibliography}

\end{document}